\documentclass[letterpaper]{article} 
\usepackage{aaai2027}  
\usepackage[hyphens]{url}  
\usepackage{graphicx} 
\usepackage{natbib}  
\usepackage{caption} 
\usepackage{amsmath}
\usepackage{amssymb}
\usepackage{amsthm}
\usepackage{booktabs}
\usepackage{multirow}
\usepackage{algorithm}
\usepackage{algorithmic}

\newcommand{\Ewarp}{\mathcal{E}_{\text{warp}}}
\newcommand{\W}{\mathcal{W}}
\newcommand{\D}{\mathcal{D}}
\newcommand{\Uhat}{\hat{U}}
\newcommand{\uhat}{\hat{u}}
\newcommand{\Ld}{L_{\mathcal{D}}}

\newtheorem{lemma}{Lemma}
\newtheorem{proposition}{Proposition}

\title{ChordVideo: One-Step, Training-Free, Temporally Consistent\\
Video Editing via Low-Energy Transport}

\author{
    Zhiqiang Lao
}

\affiliations{
    Independent Researcher
}

\begin{document}

\maketitle

\begin{abstract}
One-step text-to-image models enable training-free, inversion-free editing with only 1--2 network function evaluations (NFE), while ChordEdit stabilizes such edits through low-energy smoothing along sampling time. Applied independently to video frames, however, it produces temporal flicker and edit-strength drift. We introduce \textbf{ChordVideo}, which extends the same low-energy principle to video time through shared noise, motion-aligned causal aggregation of per-frame Chord fields, and an optional temporally smoothed proximal correction. We derive a warping-error bound that separates motion bias from stochastic flicker and predicts diminishing returns with larger temporal windows. On TGVE/DAVIS with two one-step backbones, ChordVideo reduces warping error by \textbf{78\%} and flicker by \textbf{49\%}, improves CLIP frame consistency by \textbf{9--10 points}, and increases background PSNR by about \textbf{1.5,dB}, while retaining \textbf{2 NFE/frame}. Compared with seven multi-step editors, it achieves competitive temporal consistency and source preservation using \textbf{10--60$\times$ fewer model steps per clip}.
\end{abstract}

\section{Introduction}






Diffusion distillation has enabled text-to-image models to synthesize images in a single forward pass \citep{sauer2024sdturbo,nguyen2024swiftbrush,liu2023instaflow}. ChordEdit \citep{chordedit2026} builds on these models to perform efficient image editing. It defines the naive edit field as the difference between the target- and source-conditioned score drifts:

\begin{equation}
R(x,t)=B_t\big(Q(x,t,c_{\text{tar}})-Q(x,t,c_{\text{src}})\big).
\end{equation}

Because this operation subtracts two large and potentially divergent trajectories, the resulting edit field may exhibit high energy and oscillatory behavior. Under the Benamou--Brenier dynamic optimal-transport formulation \citep{benamou2000computational}, ChordEdit decomposes the field as $R=u_t+\eta$, where $u_t$ denotes the underlying low-energy transport and $\eta$ represents zero-mean noise. It then constructs a causal two-point \emph{Chord control field} along the sampling-time axis:

\begin{equation}
\uhat_t(x)=\frac{\tau,R(x,t-\Delta)+\Delta,R(x,t)}{\tau+\Delta}.
\end{equation}

This convex combination reduces the field's $L_2$ energy by Jensen's inequality, attenuating local spikes and stabilizing the large Euler update. Consequently, ChordEdit requires only 1--2 network function evaluations, depending on whether the optional proximal correction is applied.


ChordEdit processes each image independently and therefore does not capture temporal dependencies between video frames. When applied frame by frame, this limitation leads to two failure modes (Fig.~\ref{fig:framewise-failure}). First, independently sampled noise and edit fields introduce \emph{temporal flicker}, causing appearance and structure to fluctuate across frames. Second, variations in the magnitude of $R$ produce \emph{edit-strength drift}, with some frames receiving overly strong edits and others insufficient edits. These artifacts motivate extending the low-variance, low-energy stabilization used along sampling time to the video-time dimension.

\begin{figure}[!tbp]
\centering
\includegraphics[width=0.95\columnwidth]{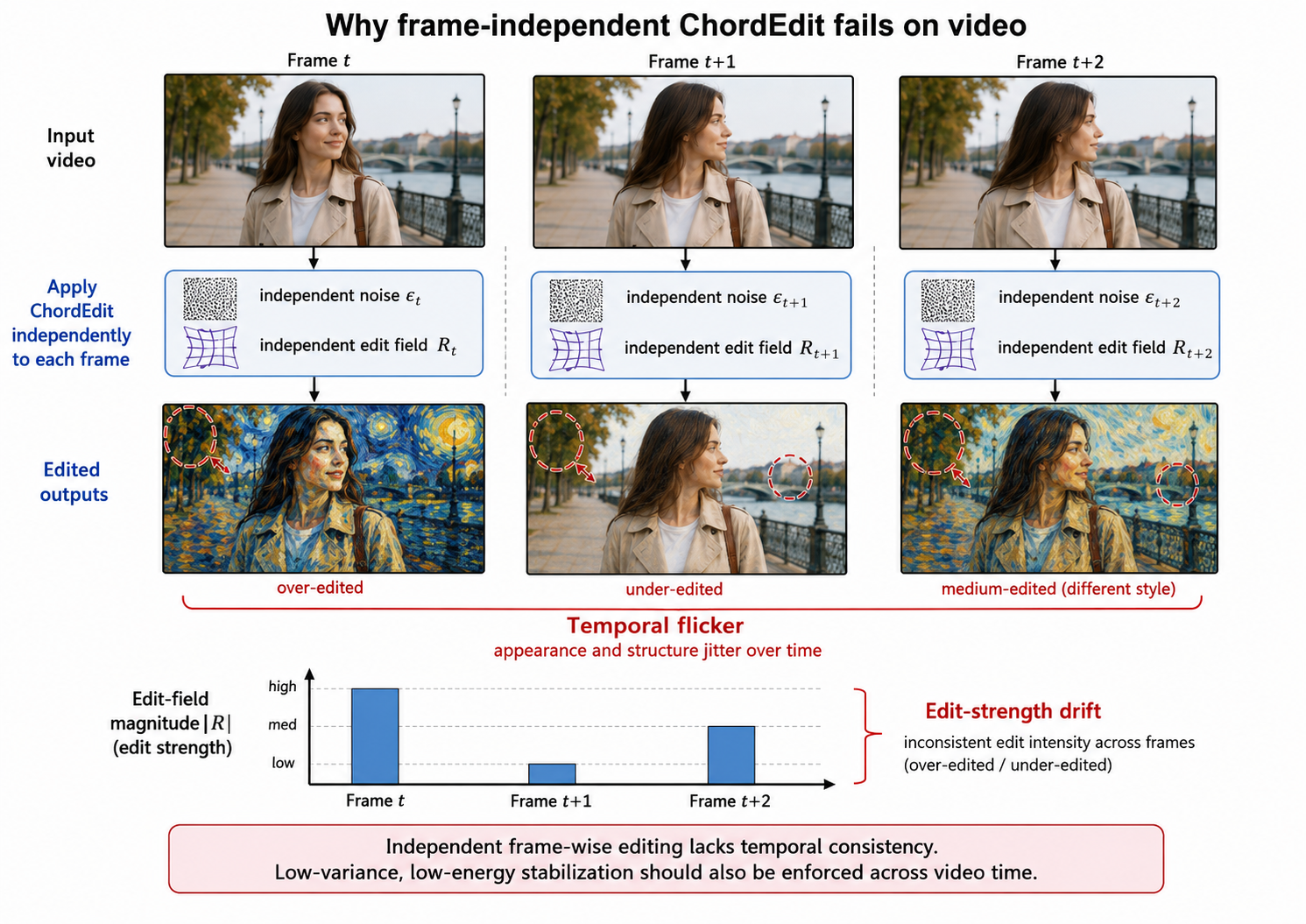}
\caption{Failure modes of frame-independent video editing. Applying ChordEdit
independently to each frame produces temporal flicker from inconsistent
appearance and structure, as well as edit-strength drift from variations in the
edit-field magnitude $\lVert R\rVert$. These artifacts motivate low-energy
stabilization across video time.}
\label{fig:framewise-failure}
\end{figure}

We extend ChordEdit’s low-energy smoothing from the sampling-time axis to the video-time axis $\tau_{\text{frame}}$. After motion alignment, edit fields from neighboring frames are interpreted as noisy measurements of a shared underlying transport signal and aggregated using a causal temporal kernel. This yields a unified optimal-transport view of stabilization across both axes: convex averaging reduces variance while limiting the bias caused by temporal changes in the underlying signal. The resulting method enables one-step, training-free video editing with explicit temporal consistency.


Our main contributions are as follows:
\begin{enumerate}



\item We introduce ChordVideo, the first framework to combine one-step, training-free, inversion-free video editing with explicit temporal stabilization (Sec.~\ref{sec:method}).

\item We derive a warping-error bound that separates motion residuals from stochastic flicker and predicts diminishing returns as the temporal window grows (Sec.~\ref{sec:theory}).

\item We unify sampling-time and video-time stabilization under a low-energy optimal-transport formulation with four ablatable modules.

\item We evaluate ChordVideo on TGVE/DAVIS with SD-Turbo and SwiftBrush-v2, including comparisons with seven multi-step editors, ablations, and failure analysis (Sec.~\ref{sec:exp}).
\end{enumerate}
\section{Related Work}

\subsection{One-step generation and editing.}
SD-Turbo \citep{sauer2024sdturbo}, SwiftBrush-v2 \citep{nguyen2024swiftbrush}, and InstaFlow \citep{liu2023instaflow} reduce conventional multi-step diffusion sampling to only one or two model evaluations. The most closely related approach is ChordEdit \citep{chordedit2026}, which enables training-free, inversion-free image editing by formulating the edit field as low-energy transport. ChordVideo preserves this image-level editing mechanism while extending it to incorporate temporal dependencies across video frames.

\subsection{Distillation objectives and optimal transport.}
One-step generators are part of a broader class of methods designed to learn straight or approximately straight transport trajectories. Consistency models \citep{song2023consistency} enforce a self-consistent mapping along the probability-flow ODE, such that states on the same trajectory are mapped to a shared endpoint. Rectified flow \citep{liu2023rectifiedflow} and flow matching \citep{lipman2022flowmatching}, by contrast, directly learn velocity fields that transport samples between noise and data distributions. ChordEdit is complementary to these methods because its dynamic optimal-transport interpretation is applied to the \emph{editing field}, rather than to the generative trajectory of the underlying backbone. As a result, it requires only the conditional drift $Q(x,t,c)$ and remains independent of the backbone’s particular distillation objective. ChordVideo retains this backbone-agnostic formulation while extending the low-energy principle from sampling time to video time.

\subsection{Text-driven video editing.}
A range of methods has been proposed for temporally consistent video editing, including Tune-A-Video \citep{wu2023tuneavideo}, TokenFlow \citep{geyer2024tokenflow}, Rerender-A-Video \citep{yang2023rerender}, Text2Video-Zero \citep{khachatryan2023text2video}, FateZero \citep{qi2023fatezero}, FLATTEN \citep{cong2024flatten}, ControlVideo \citep{zhao2023controlvideo}, Video-P2P \citep{liu2023videop2p}, and FlowDirector \citep{li2026flowdirector}. These approaches achieve temporal coherence through mechanisms such as multi-step diffusion, cross-frame attention, per-video optimization, inversion, or learned temporal modules. Consequently, they often require tens of model evaluations per frame and may incur additional inversion or tuning costs. Distilled video generators such as CausVid \citep{yin2025causvid} reduce the computational cost of video \emph{generation}; however, our experiments indicate that combining such backbones with plug-and-play feature injection \citep{tumanyan2023pnp} does not by itself ensure temporally consistent source-preserving edits. In contrast, ChordVideo imposes temporal consistency directly on the edit field. Its implicit alignment variant borrows only the nearest-neighbor correspondence strategy from TokenFlow.

\subsection{Temporal alignment and consistency metrics.}
Our default alignment module uses RAFT optical flow \citep{teed2020raft} together with forward--backward consistency masks, and temporal stability is measured using the standard occlusion-masked warping error \citep{lai2018learning}. Optical flow is not itself a contribution of this work; instead, we demonstrate both theoretically and empirically that motion-aligned low-energy averaging suppresses the variance responsible for temporal flicker.

\section{Method: ChordVideo}
\label{sec:method}

Let $V=\{x^{(1)},\dots,x^{(N)}\}$ denote the input video, with
$\{z^{(i)}\}$ representing the corresponding frame-wise VAE latent codes. ChordVideo retains ChordEdit’s training-free, inversion-free, and model-agnostic formulation, as well as its 1--2 NFE editing budget. The additional components introduced by ChordVideo operate exclusively along the temporal frame dimension. Figure~\ref{fig:chordvideo-pipeline} summarizes the complete workflow: Module~A reuses a single noise realization across all frames; each frame is then processed by ChordEdit to obtain a per-frame field; Modules~B and~C motion-align and causally aggregate neighboring fields; and Module~D optionally smooths the proximal corrections before the edited frames are decoded. The following subsections describe these components in detail.

\begin{figure*}[!t]
\centering
\includegraphics[width=0.82\textwidth]{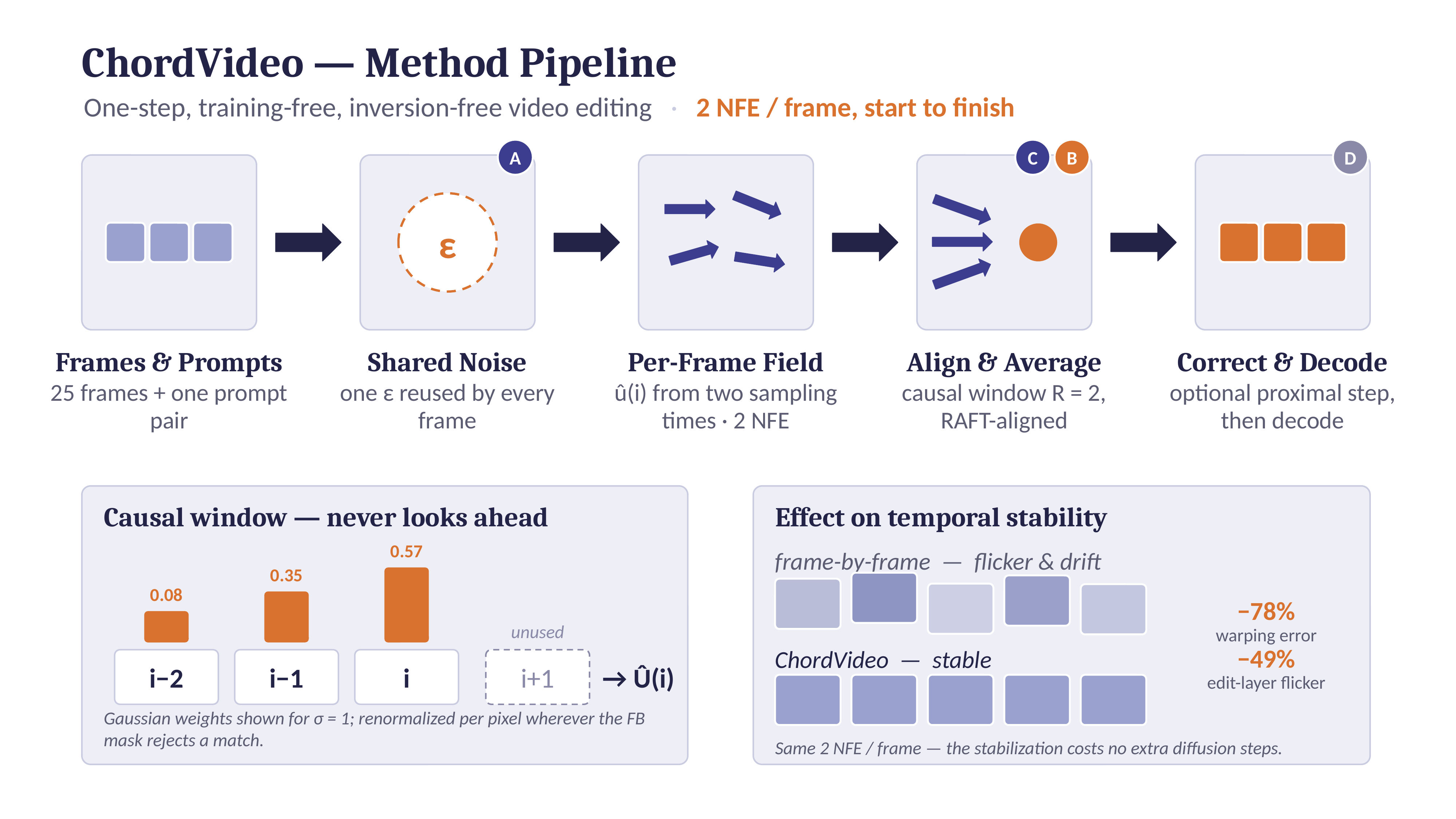}
\captionsetup{skip=0pt}
\caption{Overview of the ChordVideo pipeline. Given input frames and a source--target prompt pair, ChordVideo reuses a shared noise realization across frames (Module~A), estimates per-frame Chord fields, aligns and aggregates neighboring fields with a causal temporal kernel (Modules~B--C), and optionally applies a temporally smoothed proximal correction before decoding (Module~D). The resulting stabilization reduces warping error and flicker without increasing the 2~NFE/frame editing budget.}
\label{fig:chordvideo-pipeline}
\end{figure*}

\subsection{Background: The Chord Field}
\label{sec:background}
We begin by revisiting the origin of ChordEdit’s two-point averaging scheme, as Module~B extends the same principle to video time. Consider the naive drift
$R(x,\cdot)$ evaluated at the two available sampling-time points $t-\Delta$ and $t$.
These evaluations can be interpreted as noisy boundary measurements of an unknown continuous control signal $u(x,\cdot)$. Under the Benamou--Brenier dynamic optimal-transport formulation \citep{benamou2000computational}, the minimum-kinetic-energy path, measured by $\int\lVert\dot u\rVert^2$ , between two fixed endpoints is linear. Evaluating the resulting causal linear interpolant at time $t$ yields:
\begin{equation}
\uhat_t=\frac{\tau R(t-\Delta)+\Delta R(t)}{\tau+\Delta}.
\end{equation}
This construction has two properties that generalize naturally beyond the sampling-time axis. First, because its weights are nonnegative and sum to one, Jensen’s inequality guarantees that the averaging operation does not increase estimator variance. Second, the resulting bias is determined by the degree to which the underlying signal varies across the observations being combined. Along sampling time, this variation corresponds to changes in the edit signal; along video time, it corresponds to residual motion after alignment. Module~B extends this averaging principle to neighboring frames, while Module~C provides the correspondences required to align and compare their edit fields.

\subsection{Module A: Cross-Frame Shared Noise (Zero Cost)}
Naive frame-by-frame editing independently samples
$z^{(i)}\sim K_t(\cdot\mid x^{(i)})$, so stochastic differences appear directly
as temporal flicker. ChordVideo instead reuses a single Monte Carlo noise
realization $\{\epsilon_k\}$ for every frame. Shared noise removes one major
source of inter-frame variation at negligible computational cost and provides
the baseline on which the remaining modules build.

\subsection{Module B: Spatio-Temporal Chord Field (Core)}
We extend the sampling-time average to a causal kernel over both sampling and
video time:
\begin{equation}
\Uhat^{(i)}=\sum_{j} w_{ij}\,\W_{j\to i}\big(\uhat^{(j)}\big),\qquad
w_{ij}\ge0,\ \textstyle\sum_j w_{ij}=1.
\end{equation}
Here, $\uhat^{(j)}$ is the Chord field for frame $j$, computed using the shared
noise from Module~A. The operator $\W_{j\to i}$ aligns this field with frame
$i$, as described in Module~C. The weights $w_{ij}$ form a causal Gaussian
kernel with radius $R$ and bandwidth $\sigma$. Because the weights are
nonnegative and sum to one, the same $L_2$ contraction used by ChordEdit also
applies across video time. The aligned, temporally averaged field is then used
in a single update:
$x_{\text{pred}}^{(i)}=z^{(i)}+s\,\Uhat^{(i)}$.

\subsection{Module C: Alignment Operator $\W$}
Directly averaging fields from different frames would mix locations that refer
to different scene points, producing blur and ghosting. We therefore align each
neighboring field before aggregation. ChordVideo supports two interchangeable
alignment routes:
\begin{itemize}
\item \textbf{Optical flow (default):} RAFT estimates correspondence from frame
$j$ to frame $i$. Forward--backward (FB) consistency masks remove occluded or
unreliable regions before the field is averaged.
\item \textbf{Implicit matching:} self-similarities from the one-step backbone
or DINO features \citep{caron2021dino} are used for mutual nearest-neighbor
matching, following the correspondence strategy of TokenFlow and avoiding a
separate flow model.
\end{itemize}
For the flow route, we renormalize the temporal weights at every pixel after
masking. If all neighboring correspondences are rejected, that pixel uses its
own per-frame field. This fallback avoids propagating unreliable motion and
allows the method to degrade gracefully in occluded regions
(Sec.~\ref{sec:fail}).

\subsection{Module D: Temporally Consistent Proximal Correction (Optional)}
ChordEdit can apply an additional proximal step to strengthen the target
semantics. For frame $i$, this step produces the correction
$p^{(i)}=\text{prox}(x_{\text{pred}}^{(i)})-x_{\text{pred}}^{(i)}$. Applied
independently, however, these corrections introduce a new source of
cross-frame noise. Module~D aligns and smooths the corrections with the same
temporal kernel used in Module~B:
\begin{equation}
x_{\text{tar}}^{(i)}=x_{\text{pred}}^{(i)}+\sum_j w_{ij}\,\W_{j\to i}\big(p^{(j)}\big).
\end{equation}
The module reuses the correspondences already computed by Module~C, so it adds
no optical-flow cost. It retains the semantic benefit of the proximal update
while reducing the flicker introduced by independent per-frame corrections
(Sec.~\ref{sec:theory}).

Figure~\ref{fig:mechanism} provides a more detailed view of this motion-aligned causal aggregation. It shows how neighboring Chord fields are aligned to frame~$i$ and combined with a causal kernel, highlights the conditions underlying the Jensen-style variance contraction, and previews the empirical effect of enlarging the temporal window.

\begin{figure*}[!t]
\centering
\includegraphics[width=0.82\textwidth]{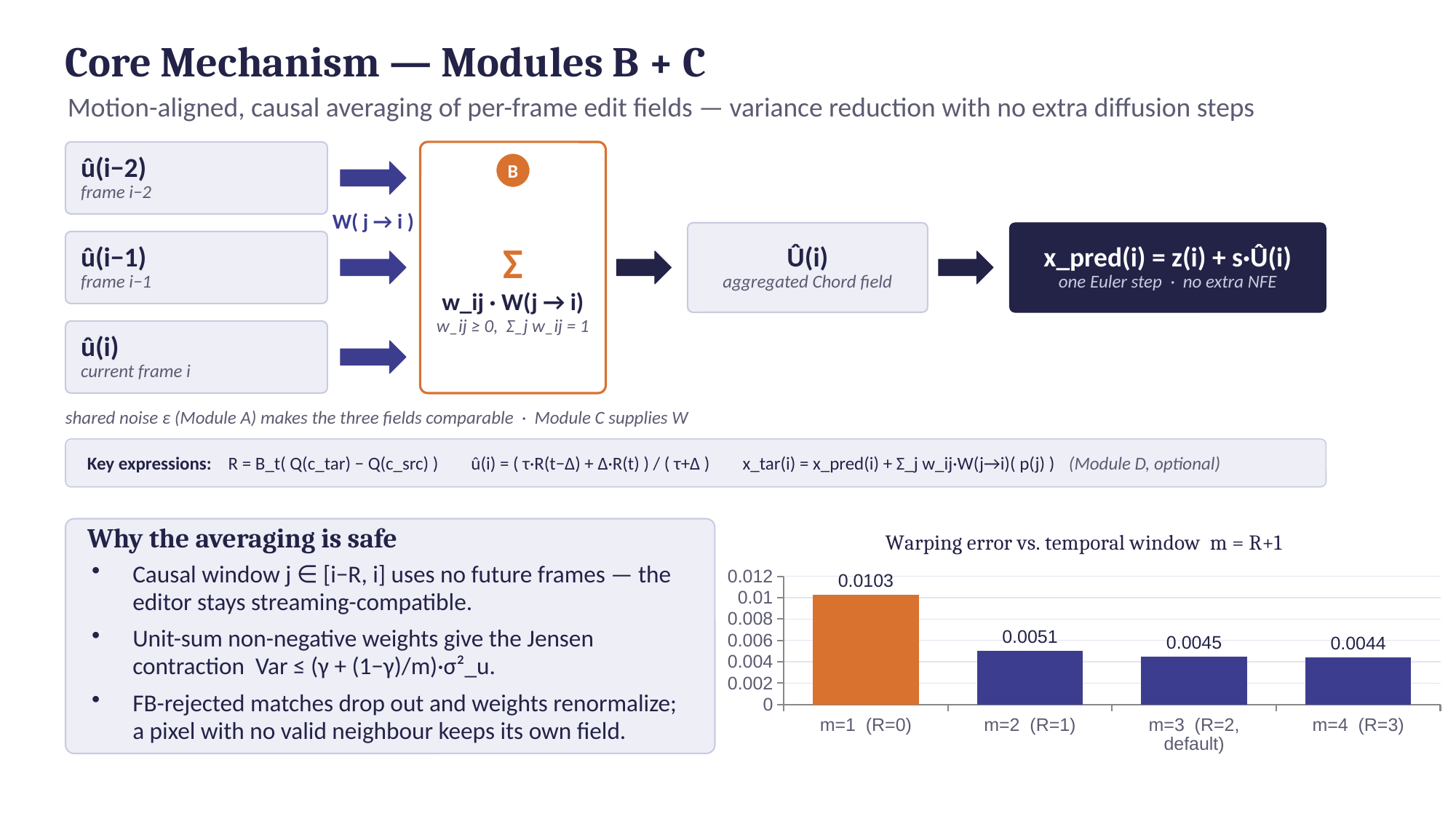}
\captionsetup{skip=0pt}
\caption{\textbf{Motion-aligned causal aggregation (Modules~B and~C).} \emph{Top:} Shared-noise Chord fields $\uhat^{(j)}$ are aligned to frame~$i$ using $\W_{j\to i}$ and combined with causal Gaussian weights $w_{ij}$ to form $\Uhat^{(i)}$, followed by one Euler step. \emph{Bottom left:} The method uses causality, nonnegative unit-sum weights, and per-pixel renormalization with self-field fallback. \emph{Bottom right:} Warping error decreases rapidly as $m=R+1$ grows, then saturates, consistent with the $\gamma+(1-\gamma)/m$ term in Lemma~\ref{lem:jensen}. Formulas correspond to Eqs.~(1)--(5).
}
\label{fig:mechanism}
\end{figure*}

\begin{algorithm}[!tbp]
\caption{ChordVideo edit procedure}
\label{alg:chordvideo}
\begin{algorithmic}[1]
\REQUIRE frames $\{x^{(i)}\}_{i=1}^N$, prompts $c_{\text{src}}, c_{\text{tar}}$, window
radius $R$, kernel $\sigma$, scale $s$
\STATE Encode latents $z^{(i)} \leftarrow \text{VAE-Encode}(x^{(i)})$
\STATE Sample \emph{one} shared noise draw $\epsilon$ (Module~A)
\FOR{$i = 1$ to $N$}
  \STATE $\uhat^{(i)} \leftarrow$ ChordEdit-Field$(z^{(i)}, \epsilon, c_{\text{src}}, c_{\text{tar}})$
  \COMMENT{2 NFE: transport $+$ optional proximal}
\ENDFOR
\FOR{$i = 1$ to $N$}
  \FOR{$j$ in causal window $[i{-}R, i]$}
    \STATE $\mathcal{W}_{j\to i} \leftarrow$ RAFT flow $+$ FB occlusion mask (Module~C)
  \ENDFOR
  \STATE $\Uhat^{(i)} \leftarrow \sum_j w_{ij}\,\mathcal{W}_{j\to i}(\uhat^{(j)})$
  \COMMENT{Module~B, causal Gaussian $w$}
  \STATE $x_{\text{pred}}^{(i)} \leftarrow z^{(i)} + s\,\Uhat^{(i)}$
  \IF{proximal enabled}
    \STATE $p^{(j)} \leftarrow \text{prox}(x_{\text{pred}}^{(j)}) - x_{\text{pred}}^{(j)}$ for
    $j$ in window \COMMENT{reuses $\mathcal{W}$}
    \STATE $x_{\text{tar}}^{(i)} \leftarrow x_{\text{pred}}^{(i)} + \sum_j w_{ij}\,\mathcal{W}_{j\to i}(p^{(j)})$
    \COMMENT{Module~D}
  \ELSE
    \STATE $x_{\text{tar}}^{(i)} \leftarrow x_{\text{pred}}^{(i)}$
  \ENDIF
\ENDFOR
\STATE \textbf{return} $\{\text{VAE-Decode}(x_{\text{tar}}^{(i)})\}_{i=1}^N$
\end{algorithmic}
\end{algorithm}

\section{Theory: A Warping-Error Bound}
\label{sec:theory}

We model the per-frame edit drift as $R(z,t)=u_t(z)+\eta$, where $u_t$ is the
underlying low-energy signal and $\mathbb{E}[\eta]=0$. The decoded edit layer for
frame $i$ is
$e^{(i)}=\D(z^{(i)}+s\Uhat^{(i)})-\D(z^{(i)})$. We measure temporal stability
with the occlusion-masked warping error
\begin{equation}
\Ewarp=\frac{1}{N-1}\sum_i\big\lVert M_{i,i+1}\odot(e^{(i+1)}-\W_{i\to i+1}e^{(i)})\big\rVert_2^2.
\end{equation}

\paragraph{Assumptions.}
We assume that (A1) the decoder is $\Ld$-Lipschitz; (A2) the clean edit field is
warp-consistent up to a residual $\rho$; (A3) the per-frame residual has
variance $\sigma_u^2$, and its post-alignment cross-frame correlation is at
most $\gamma<1$; and (A4) for a fixed correspondence and mask, the warp is
linear and non-expansive.

\begin{lemma}[Video-time Jensen contraction]
\label{lem:jensen}
With $\bar U^{(i)}=\sum_j w_{ij}\W_{j\to i}(u^{(j)})$,
\begin{equation}
\mathbb{E}\lVert\Uhat^{(i)}-\bar U^{(i)}\rVert^2\le
\big(\gamma+(1-\gamma)\lVert w_{i\cdot}\rVert_2^2\big)\sigma_u^2=:\Phi(w)\,\sigma_u^2.
\end{equation}
For a flat (causal) window of size $m$, $\lVert w_{i\cdot}\rVert_2^2=1/m$, so
$\Phi=\gamma+(1-\gamma)/m\to\gamma$ as $m\uparrow$.
\end{lemma}

\begin{proof}[Proof sketch]
Let $e_j=\W_{j\to i}(\uhat^{(j)}-u^{(j)})$. By the linearity of the warp,
$\Uhat^{(i)}-\bar U^{(i)}=\sum_j w_{ij}e_j$. Non-expansiveness gives
$\mathbb{E}\lVert e_j\rVert^2\le\sigma_u^2$, while Assumption~(A3) bounds each
cross term by
$\mathbb{E}\langle e_j,e_k\rangle\le\gamma\sigma_u^2$ for $j\ne k$. Expanding
the squared norm yields
\begin{align*}
\mathbb{E}\Big\lVert\sum_j w_{ij}e_j\Big\rVert^2
&=\sum_j w_{ij}^2\mathbb{E}\lVert e_j\rVert^2
 +\sum_{j\ne k}w_{ij}w_{ik}\mathbb{E}\langle e_j,e_k\rangle\\
&\le\sigma_u^2\!\left[\lVert w_{i\cdot}\rVert_2^2
 +\gamma\big(1-\lVert w_{i\cdot}\rVert_2^2\big)\right].
\end{align*}
The identity
$\sum_{j\ne k}w_{ij}w_{ik}=1-\lVert w_{i\cdot}\rVert_2^2$ follows from the
unit-sum weights. Rearranging gives
$\Phi(w)=\gamma+(1-\gamma)\lVert w_{i\cdot}\rVert_2^2$.
\end{proof}

\begin{proposition}[Warping-error bound]
\label{prop:bound}
Under (A1)--(A4) with flat causal window $m=R+1$,
\begin{equation}
\Ewarp\le\underbrace{2\Ld^2 s^2\rho^2}_{\text{bias (motion residual)}}
+\underbrace{2\Ld^2 s^2\big(\gamma+\tfrac{1-\gamma}{m}\big)\sigma_u^2}_{\text{variance (flicker)}}.
\end{equation}
\end{proposition}

\begin{proof}[Proof sketch]
Assumption~(A1) gives
\begin{equation*}
\lVert e^{(i+1)}-\W_{i\to i+1}e^{(i)}\rVert
\le \Ld s\lVert U^{(i+1)}-\W_{i\to i+1}U^{(i)}\rVert,
\end{equation*}
where $U^{(i)}\equiv\Uhat^{(i)}$. We decompose the field difference around the
aligned clean field $\bar U$:
\begin{align*}
U^{(i+1)}-\W U^{(i)}
={}&(\Uhat^{(i+1)}-\bar U^{(i+1)})
 +(\bar U^{(i+1)}-\W\bar U^{(i)})\\
&-\W(\Uhat^{(i)}-\bar U^{(i)}).
\end{align*}
The middle term is the motion-induced bias and has norm at most $\rho$ by
Assumption~(A2). The two remaining terms are stochastic deviations, each
controlled by Lemma~\ref{lem:jensen}; Assumption~(A4) ensures that warping does
not increase their norms. Applying the standard inequality
$\lVert a+b\rVert^2\le2\lVert a\rVert^2+2\lVert b\rVert^2$, taking expectations,
and averaging over frames gives
$\Ewarp\le2\Ld^2s^2\rho^2+2\Ld^2s^2\Phi(w)\sigma_u^2$. Substituting
$\Phi(w)=\gamma+(1-\gamma)/m$ proves the result. The supplementary material
tracks the constants associated with the occlusion mask $M$.
\end{proof}

\noindent\textbf{Consequences.}
The bound predicts diminishing variance reduction as $m$ grows, with a floor at $2\Ld^2s^2(\rho^2+\gamma\sigma_u^2)$. Motion alignment is essential to avoid bias and ghosting, while larger windows may slightly weaken edits through temporal smoothing. The same trade-off applies to Module~D; full proofs appear in the supplement.

\section{Experiments}
\label{sec:exp}

\paragraph{Setup.}
We evaluate SD-Turbo and SwiftBrush-v2 on LOVEU-TGVE-2023 \citep{loveu2023tgve} using 25 uniformly sampled DAVIS frames per clip at $512\times512$. Experiments use one GPU, fp16, and seed~42. Unless noted, ChordVideo uses shared noise, $R{=}2$, $\sigma{=}1$, a causal Gaussian kernel, RAFT alignment, and temporal proximal correction.

\paragraph{Evaluation protocol.}
We evaluate four aspects: temporal consistency using occlusion-masked warping error, edit-layer flicker, and CLIP frame consistency \citep{radford2021clip}; edit quality using target-prompt similarity and the target--source CLIP gap; background fidelity using whole-frame PSNR and MSE; and efficiency using NFE/frame, throughput, and peak VRAM.

\paragraph{Internal baselines.}
All internal variants use the same backbone, input frames, and random seed.
\texttt{frame\_by\_frame} applies ChordEdit independently to each frame with
independent noise. \texttt{shared\_noise} adds only Module~A.
\texttt{chordvideo} uses the full A+B+C+D configuration.

\paragraph{State-of-the-art comparison.}
For Sec.~\ref{sec:sota}, we evaluate seven external editors using official implementations and recommended settings on the same seven DAVIS clips, frames, and prompts. We also include 50-step SDEdit \citep{meng2021sdedit}, two-step SD-Turbo img2img, and CausVid with PnP injection \citep{yin2025causvid,tumanyan2023pnp}. All methods use identical implementations of RAFT warping error, CLIP-L/14 and directional similarity, edit-layer flicker, and mean pixel change.

\paragraph{Implementation details.}
Table~\ref{tab:hyper} lists the fixed hyperparameters for both backbones and all clips. Here, $t_{\text{start}}$ and $t_\Delta$ define the sampling times, $R$ and $\sigma$ the temporal kernel, $s$ the Euler scale, and the FB threshold the valid Module~C correspondences.

\begin{table}[!tbp]
\centering
\scriptsize
\begin{tabular}{ll}
\toprule
Hyperparameter & Default \\
\midrule
Sampling start / step $t_{\text{start}}, t_\Delta$ & $0.90,\ 0.15$ \\
Step scale $s$ & $1.0$ \\
Window radius $R$ & $2$ \\
Kernel bandwidth $\sigma$ & $1.0$ \\
Kernel shape & causal Gaussian \\
Alignment route & RAFT flow \\
FB occlusion threshold & $1.5$\,px \\
Proximal mode & temporal (Module~D) \\
Frames / clip & 25 \\
Resolution & $512{\times}512$ \\
Seed & 42 \\
\bottomrule
\end{tabular}
\captionsetup{font=scriptsize}
\caption{Default ChordVideo hyperparameters, fixed across all backbones and clips.}
\label{tab:hyper}
\end{table}

Table~\ref{tab:clips} lists the seven DAVIS clips and the source--target prompt pairs
used throughout Secs.~\ref{sec:exp}--\ref{sec:sota}. Each method receives the same 25
uniformly sampled frames and the same prompt pair for every clip.

\begin{table}[!tbp]
\centering
\scriptsize
\setlength{\tabcolsep}{3pt}
\begin{tabular}{lll}
\toprule
Clip & Source prompt & Target prompt \\
\midrule
bear & a bear walking & a polar bear walking \\
blackswan & a black swan swimming & a white swan swimming \\
 & in a pond & in a pond \\
goldfish & goldfish swimming in a & blue fish swimming in a \\
 & fish tank & fish tank \\
mbike & a motorcyclist doing a & a motorcyclist on a white \\
 & trick on a road & motorcycle doing a trick \\
drift & a car drifting on a & a red car drifting on a \\
 & race track & race track \\
lindyhop & people dancing in a dance & people dancing in a dance \\
 & hall, wooden floor & hall, red floor \\
swimmer & a person diving into a & a person diving into a \\
 & swimming pool & green swimming pool \\
\bottomrule
\end{tabular}
\captionsetup{font=scriptsize}
\caption{Evaluation clips and source--target prompts. Each clip contains 25 sampled
frames at $512\times512$ resolution.}
\label{tab:clips}
\end{table}

\begin{table*}[!t]
\centering
\scriptsize
\setlength{\tabcolsep}{4.5pt}
\begin{tabular}{lccrrrrrrr}
\toprule
Method & TF & IF & Steps/f & CLIP-T $\uparrow$ & CLIP-F $\uparrow$ & Warp-E $\downarrow$ & Flicker $\downarrow$ & CDS $\uparrow$ & Pix$\Delta$ $\downarrow$ \\
\midrule
TokenFlow \citep{geyer2024tokenflow} & \checkmark & $\times$ & 50+50 & 0.2576 & 0.9744 & \textbf{0.0019} & \textbf{0.0375} & 0.1375 & 20.7 \\
FlowDirector \citep{li2026flowdirector} & \checkmark & \checkmark & 50 & 0.2579 & 0.9683 & 0.0054 & 0.0436 & 0.1620 & \textbf{19.7} \\
FateZero \citep{qi2023fatezero} & \checkmark & $\times$ & 10+10 & 0.1612 & \textbf{0.9926} & 0.0262 & 0.0838 & 0.0285 & 116.2 \\
FLATTEN \citep{cong2024flatten} & \checkmark & $\times$ & 50+50 & 0.2575 & 0.9792 & 0.0110 & 0.0706 & \textbf{0.1700} & 30.9 \\
ControlVideo \citep{zhao2023controlvideo} & $\times$ (300) & $\times$ & 50 & 0.2659 & 0.9840 & 0.0213 & 0.0752 & 0.1475 & 41.1 \\
Tune-A-Video \citep{wu2023tuneavideo} & $\times$ (300) & $\times$ & 50+50 & 0.2602 & 0.9852 & 0.0344 & 0.0820 & 0.1298 & 71.9 \\
Video-P2P \citep{liu2023videop2p}$^\dagger$ & $\times$ (500) & $\times$ & 50+50 & \textbf{0.2714} & 0.9824 & 0.0359 & 0.0809 & 0.1203 & 38.8 \\
\midrule
SDEdit SD1.5 per-frame \citep{meng2021sdedit} & \checkmark & \checkmark & 50 & 0.2563 & 0.9450 & 0.0308 & 0.0927 & 0.1226 & 34.7 \\
CausVid $+$ PnP \citep{yin2025causvid} & \checkmark & $\times$ & --- & 0.2301 & 0.9762 & 0.0324 & 0.0727 & 0.0604 & 37.6 \\
SD-Turbo naive per-frame & \checkmark & \checkmark & 2 & 0.2587 & 0.9580 & 0.0244 & 0.0944 & 0.1609 & 48.7 \\
\midrule
ChordEdit per-frame (ours, base) & \checkmark & \checkmark & \textbf{2} & 0.2407 & 0.9026 & 0.0290 & 0.1121 & 0.0720 & 23.9 \\
\textbf{ChordVideo (ours, full)} & \checkmark & \checkmark & \textbf{2} & 0.2525 & 0.9241 & 0.0052 & 0.0538 & 0.1584 & 23.8 \\
\bottomrule
\end{tabular}
\captionsetup{font=scriptsize}
\caption{
Comparison with state-of-the-art editors on seven DAVIS clips using identical frames, prompts, and metrics. TF denotes training-free operation, IF inversion-free operation, and Steps/f inversion plus sampling steps per frame. $^\dagger$Video-P2P completes five clips; on this subset, ChordVideo achieves CLIP-T 0.2569 versus 0.2566 for TokenFlow and 0.2564 for FlowDirector.}
\label{tab:sota}
\end{table*}

\subsection{Main Results}

\begin{figure}[!tbp]
\centering
\includegraphics[width=\linewidth]{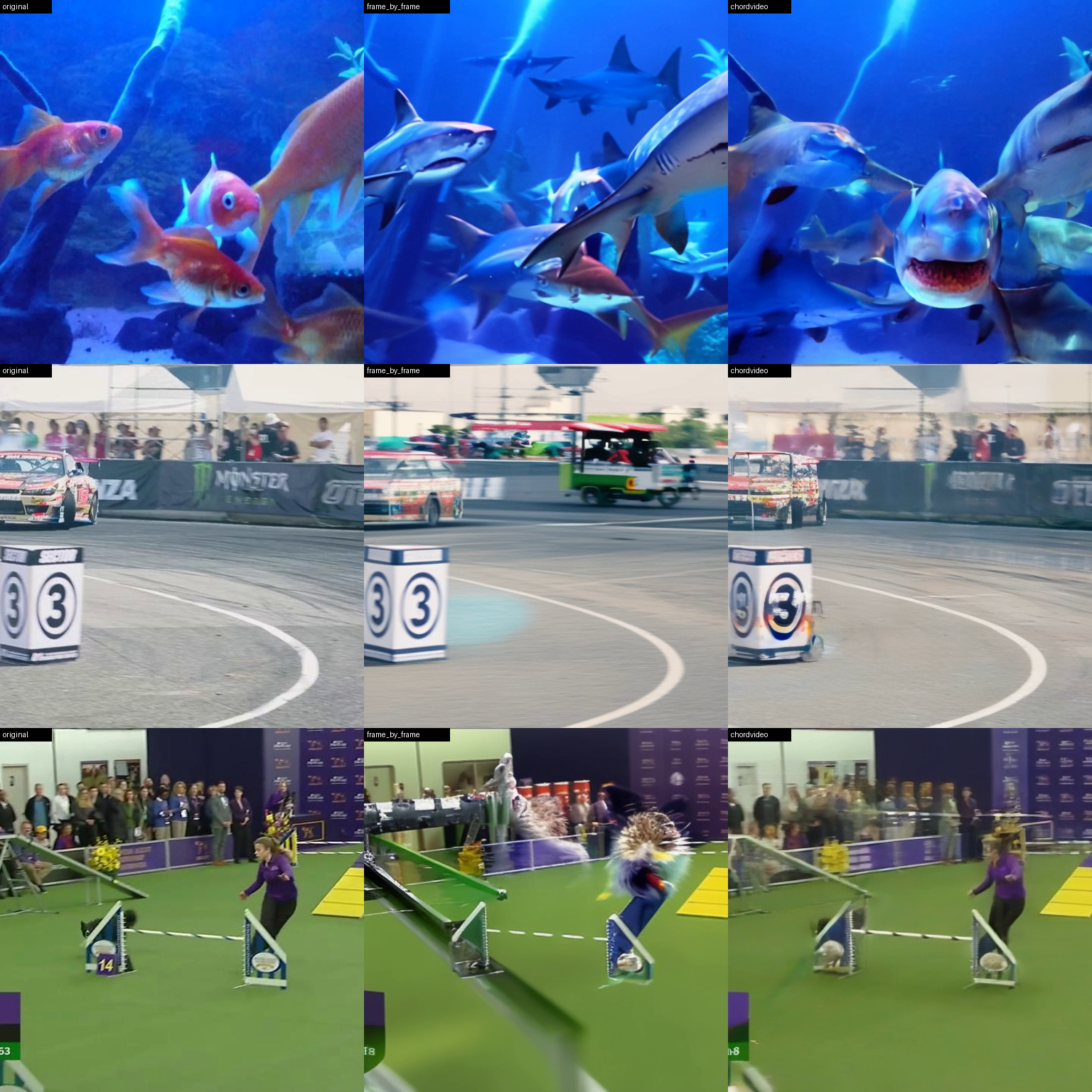}
\caption{Qualitative comparison on three clips. Columns show the source, per-frame ChordEdit, and ChordVideo. ChordVideo better preserves source geometry and target edits, while per-frame editing exhibits temporal drift. Additional videos and methods appear in the supplement.}
\label{fig:teaser}
\end{figure}


\begin{table}[!tbp]
\centering
{\scriptsize
\setlength{\tabcolsep}{2.5pt}
\renewcommand{\arraystretch}{0.90}
\begin{tabular}{llrrr}
\toprule
Cat. & Metric & \texttt{f-by-f} & \texttt{shared} & \textbf{\texttt{CV}} \\
\midrule
\multirow{3}{*}{Temp.}
 & Warp $\downarrow$   & 0.02784 & 0.01419 & \textbf{0.00610} \\
 & Flicker $\downarrow$ & 0.11487 & 0.07473 & \textbf{0.05806} \\
 & CLIP-tc $\uparrow$   & 0.8181  & 0.8792  & \textbf{0.9053} \\
\midrule
Qual. & CLIP-tgt $\uparrow$ & \textbf{0.2547} & 0.2548 & 0.2421 \\
BG    & PSNR $\uparrow$     & 17.30 & 17.74 & \textbf{19.02} \\
\midrule
\multirow{3}{*}{Eff.}
 & NFE $\downarrow$  & 2    & 2    & \textbf{2} \\
 & FPS $\uparrow$    & 7.47 & 7.77 & 3.91 \\
 & VRAM $\downarrow$ & 4549 & 4840 & 4956 \\
\bottomrule
\end{tabular}
}
\captionsetup{font=scriptsize}
\caption{SD-Turbo results on six clips. CV denotes ChordVideo; VRAM is reported in MB.}
\label{tab:main_sdturbo}
\end{table}


\begin{table}[!tbp]
\centering
{\scriptsize
\setlength{\tabcolsep}{2.5pt}
\renewcommand{\arraystretch}{0.90}
\begin{tabular}{llrrr}
\toprule
Cat. & Metric & \texttt{f-by-f} & \texttt{shared} & \textbf{\texttt{CV}} \\
\midrule
\multirow{3}{*}{Temp.}
 & Warp $\downarrow$    & 0.02326 & 0.01151 & \textbf{0.00501} \\
 & Flicker $\downarrow$ & 0.11279 & 0.07555 & \textbf{0.05715} \\
 & CLIP-tc $\uparrow$   & 0.8032  & 0.8521  & \textbf{0.9017} \\
\midrule
Qual. & CLIP-tgt $\uparrow$ & 0.2460 & \textbf{0.2470} & 0.2375 \\
BG    & PSNR $\uparrow$     & 18.04  & 18.34 & \textbf{19.56} \\
\midrule
\multirow{3}{*}{Eff.}
 & NFE $\downarrow$  & 2    & 2    & \textbf{2} \\
 & FPS $\uparrow$    & 7.59 & 7.66 & 4.13 \\
 & VRAM $\downarrow$ & 4844 & 4844 & 4960 \\
\bottomrule
\end{tabular}
}
\captionsetup{font=scriptsize}
\caption{SwiftBrush-v2 results on six clips. CV denotes ChordVideo; VRAM is in MB.}
\label{tab:main_sbv2}
\end{table}

Across both backbones, ChordVideo reduces warping error by about $78\%$ and flicker by $49\%$, improves CLIP temporal consistency by 9--10 points and PSNR by $1.5$--$1.7$,dB, while retaining 2~NFE/frame. These gains come with lower throughput, slightly higher memory use, and a 0.01--0.02 drop in CLIP target similarity, largely due to the difficult drift-turn edit.

\subsection{Comparison with State-of-the-Art Editors}
\label{sec:sota}

\paragraph{Pareto analysis.}

Table~\ref{tab:sota} shows that ChordVideo is the only method that is training-free, inversion-free, and limited to 2~NFE/frame. Competing editors require 10--50$\times$ more sampling or inversion steps, while tuning-based methods add 300--500 optimization steps per video; Table~\ref{tab:cost} summarizes these clip-level costs.

Despite this budget, ChordVideo achieves a warping error of 0.0052, comparable to FlowDirector's 0.0054 and better than five of seven multi-step editors. TokenFlow reaches 0.0019 but uses inversion and roughly 50$\times$ more evaluations. Because CLIP-F can reward smooth but unsuccessful edits, we report it alongside edit quality and warping error. Overall, ChordVideo offers a strong efficiency--consistency trade-off.

\begin{table}[!tbp]
\centering
\scriptsize
\setlength{\tabcolsep}{4pt}
\begin{tabular}{lrrrr}
\toprule
Method & Tune & /Frame & Total & $\times$CV \\
\midrule
ChordVideo (ours) & 0 & 2 & 50 & \textbf{1$\times$} \\
SD-Turbo naive (ref) & 0 & 2 & 50 & 1$\times$ \\
FateZero & 0 & 20 & 500 & 10$\times$ \\
FlowDirector & 0 & 50 & 1250 & 25$\times$ \\
SDEdit SD1.5 (ref) & 0 & 50 & 1250 & 25$\times$ \\
ControlVideo & 300 & 50 & 1550 & 31$\times$ \\
TokenFlow & 0 & 100 & 2500 & 50$\times$ \\
FLATTEN & 0 & 100 & 2500 & 50$\times$ \\
Tune-A-Video & 300 & 100 & 2800 & 56$\times$ \\
Video-P2P & 500 & 100 & 3000 & \textbf{60$\times$} \\
\bottomrule
\end{tabular}
\captionsetup{font=scriptsize}
\caption{
Counted steps for a 25-frame clip, including per-video tuning and per-frame inversion/sampling from Table~\ref{tab:sota}. ``$\times$CV'' denotes the ratio to ChordVideo's 50 steps; external editors require 10--60$\times$ more.
}
\label{tab:cost}
\end{table}

\subsection{Module Ablations}

\begin{table}[!tbp]
\centering
\small
\setlength{\tabcolsep}{4pt}
\begin{tabular}{lrrrr}
\toprule
Config & warp $\downarrow$ & flick $\downarrow$ & CLIP-tc $\uparrow$ & CLIP-tgt $\uparrow$ \\
\midrule
\textbf{full (A+B+C+D)} & \textbf{0.00447} & \textbf{0.0482} & 0.9201 & 0.2378 \\
A off: indep.\ noise & 0.00732 & 0.0639 & 0.8776 & 0.2313 \\
B $R{=}0$ (off) & 0.01028 & 0.0642 & 0.8937 & 0.2415 \\
B $R{=}1$ & 0.00505 & 0.0504 & 0.9170 & 0.2393 \\
B $R{=}3$ & 0.00443 & 0.0481 & 0.9206 & 0.2376 \\
C none & 0.00585 & 0.0478 & 0.9296 & 0.2360 \\
C implicit & 0.00642 & 0.0549 & 0.9080 & 0.2396 \\
D none (field only) & 0.00592 & 0.0483 & 0.9518 & \textbf{0.2026} \\
D shared (per-frame) & 0.00596 & 0.0533 & 0.8988 & 0.2414 \\
\bottomrule
\end{tabular}
\captionsetup{font=scriptsize}
\caption{Module ablation on four clips with SD-Turbo. Each row changes one component
relative to the full A+B+C+D configuration.}
\label{tab:ablation}
\end{table}

Table~\ref{tab:ablation} shows that Module~B yields the largest gain: disabling aggregation raises warping error by $130\%$, while removing shared noise raises it by $64\%$. Gains saturate beyond $R{=}2$. RAFT performs best, and temporal smoothing in Module~D preserves semantics while minimizing flicker and warping error.

\begin{table}[!tbp]
\centering
\scriptsize
\setlength{\tabcolsep}{4pt}
\begin{tabular}{lrrrr}
\toprule
Clip & Occl.\ $\%$ & Flow (px) & Warp+FB & Warp$-$FB \\
\midrule
mbike-trick & 1.3 & 0.9 & 0.00211 & 0.00211 \\
drift-turn & 2.4 & 52.5 & 0.00667 & 0.00645 \\
lindy-hop & 12.1 & 3.4 & 0.00766 & 0.00697 \\
swimmer & 47.2 & 20.0 & 0.01207 & 0.01152 \\
\bottomrule
\end{tabular}
\captionsetup{font=scriptsize}
\caption{Failure-case indicators for SD-Turbo. Occlusion fraction and mean flow
magnitude describe clip difficulty; Warp+FB and Warp$-$FB report warping error with and
without masked weight renormalization.}
\label{tab:failure}
\end{table}

\subsection{Failure Cases \& Fallback}
\label{sec:fail}




Table~\ref{tab:failure} shows distinct challenges from occlusion and motion. Masked renormalization slightly raises measured warping error but reduces ghosting by rejecting unreliable correspondences; examples are provided in the supplement.

\section{Positioning \& Limitations}

\paragraph{Generality across backbones.}
Despite different distillation objectives and U-Net parameters, SD-Turbo and SwiftBrush-v2 show similar gains: about $78\%$ lower warping error, $49\%$ lower flicker, and comparable CLIP consistency and PSNR improvements. Because Modules~A--D act on the edit field $R(x,t)$ and decoder $\D$ rather than backbone parameters, ChordVideo can extend to any one-step editor with prompt-conditioned drift.

\paragraph{Computational overhead.}
ChordVideo keeps 2~NFE/frame, but RAFT reduces throughput from about 7.5 to 4 FPS. Lighter flow or stronger implicit matching could reduce this cost.


\paragraph{Novelty beyond a direct extension.}
ChordVideo analyzes the bias--variance trade-off, enables training-free and inversion-free editing at 2~NFE/frame, and unifies sampling- and video-time stabilization under low-energy optimal transport.





\paragraph{Scope of the claims.}
ChordVideo does not lead every metric, but offers competitive stability and preservation at 2~NFE/frame without inversion or per-video training.

\paragraph{Limitations.}
RAFT lowers throughput to about 4 FPS, large motion may weaken edits, and CLIP consistency trails some multi-step methods. Evaluation is limited to six or seven clips on one non-batched GPU, and PSNR is only a preservation proxy.


\paragraph{Future work.}
Future work will integrate low-energy constraints into video backbones, adapt $R$ using occlusion, and expand TGVE and human evaluation.

\section{Conclusion}
ChordVideo extends low-energy smoothing to video, achieving one-step, training-free temporal editing. Across two backbones, it reduces warping error by $78\%$ and flicker by $49\%$ at 2~NFE/frame, offering competitive consistency and source preservation without inversion or per-video training.
\bibliography{references}

@inproceedings{chordedit2026,
  title     = {{ChordEdit}: One-Step Low-Energy Transport for Image Editing},
  author    = {Lu, Liangsi and Chen, Xuhang and Guo, Minzhe and Li, Shichu and Wang, Jingchao and Shi, Yang},
  booktitle = {IEEE/CVF Conference on Computer Vision and Pattern Recognition (CVPR)},
  year      = {2026}
}

@inproceedings{sauer2024sdturbo,
  title     = {Adversarial Diffusion Distillation},
  author    = {Sauer, Axel and Lorenz, Dominik and Blattmann, Andreas and Rombach, Robin},
  booktitle = {European Conference on Computer Vision (ECCV)},
  year      = {2024}
}

@inproceedings{nguyen2024swiftbrush,
  title     = {{SwiftBrush} v2: Make Your One-Step Diffusion Model Better Than Its Teacher},
  author    = {Nguyen, Trung Tuan and Dao, Quan and Phung, Dinh and Tran, Anh},
  booktitle = {European Conference on Computer Vision (ECCV)},
  year      = {2024}
}

@inproceedings{liu2023instaflow,
  title     = {{InstaFlow}: One Step is Enough for High-Quality Diffusion-Based Text-to-Image Generation},
  author    = {Liu, Xingchao and Zhang, Xiwen and Ma, Jianzhu and Peng, Jian and Liu, Qiang},
  booktitle = {International Conference on Learning Representations (ICLR)},
  year      = {2023}
}

@inproceedings{teed2020raft,
  title     = {{RAFT}: Recurrent All-Pairs Field Transforms for Optical Flow},
  author    = {Teed, Zachary and Deng, Jia},
  booktitle = {European Conference on Computer Vision (ECCV)},
  year      = {2020}
}

@inproceedings{lai2018learning,
  title     = {Learning Blind Video Temporal Consistency},
  author    = {Lai, Wei-Sheng and Huang, Jia-Bin and Wang, Oliver and Shechtman, Eli and Yumer, Ersin and Yang, Ming-Hsuan},
  booktitle = {European Conference on Computer Vision (ECCV)},
  year      = {2018}
}

@inproceedings{geyer2024tokenflow,
  title     = {{TokenFlow}: Consistent Diffusion Features for Consistent Video Editing},
  author    = {Geyer, Michal and Bar-Tal, Omer and Bagon, Shai and Dekel, Tali},
  booktitle = {International Conference on Learning Representations (ICLR)},
  year      = {2024}
}

@inproceedings{wu2023tuneavideo,
  title     = {{Tune-A-Video}: One-Shot Tuning of Image Diffusion Models for Text-to-Video Generation},
  author    = {Wu, Jay Zhangjie and Ge, Yixiao and Wang, Xintao and Lei, Stan Weixian and Gu, Yuchao and Shi, Yufei and Hsu, Wynne and Shan, Ying and Qie, Xiaohu and Shou, Mike Zheng},
  booktitle = {IEEE/CVF International Conference on Computer Vision (ICCV)},
  year      = {2023}
}

@inproceedings{qi2023fatezero,
  title     = {{FateZero}: Fusing Attentions for Zero-Shot Text-Based Video Editing},
  author    = {Qi, Chenyang and Cun, Xiaodong and Zhang, Yong and Lei, Chenyang and Wang, Xintao and Shan, Ying and Chen, Qifeng},
  booktitle = {IEEE/CVF International Conference on Computer Vision (ICCV)},
  year      = {2023}
}

@inproceedings{khachatryan2023text2video,
  title     = {{Text2Video-Zero}: Text-to-Image Diffusion Models Are Zero-Shot Video Generators},
  author    = {Khachatryan, Levon and Movsisyan, Andranik and Tadevosyan, Vahram and Henschel, Roberto and Wang, Zhangyang and Navasardyan, Shant and Shi, Humphrey},
  booktitle = {IEEE/CVF International Conference on Computer Vision (ICCV)},
  year      = {2023}
}

@inproceedings{yang2023rerender,
  title     = {Rerender A Video: Zero-Shot Text-Guided Video-to-Video Translation},
  author    = {Yang, Shuai and Zhou, Yifan and Liu, Ziwei and Loy, Chen Change},
  booktitle = {SIGGRAPH Asia},
  year      = {2023}
}

@article{benamou2000computational,
  title     = {A Computational Fluid Mechanics Solution to the {Monge-Kantorovich} Mass Transfer Problem},
  author    = {Benamou, Jean-David and Brenier, Yann},
  journal   = {Numerische Mathematik},
  volume    = {84},
  number    = {3},
  pages     = {375--393},
  year      = {2000}
}

@inproceedings{caron2021dino,
  title     = {Emerging Properties in Self-Supervised Vision Transformers},
  author    = {Caron, Mathilde and Touvron, Hugo and Misra, Ishan and J{\'e}gou, Herv{\'e} and Mairal, Julien and Bojanowski, Piotr and Joulin, Armand},
  booktitle = {IEEE/CVF International Conference on Computer Vision (ICCV)},
  year      = {2021}
}

@inproceedings{radford2021clip,
  title     = {Learning Transferable Visual Models From Natural Language Supervision},
  author    = {Radford, Alec and Kim, Jong Wook and Hallacy, Chris and Ramesh, Aditya and Goh, Gabriel and Agarwal, Sandhini and Sastry, Girish and Askell, Amanda and Mishkin, Pamela and Clark, Jack and Krueger, Gretchen and Sutskever, Ilya},
  booktitle = {International Conference on Machine Learning (ICML)},
  year      = {2021}
}

@article{loveu2023tgve,
  title={Cvpr 2023 text guided video editing competition},
  author={Wu, Jay Zhangjie and Li, Xiuyu and Gao, Difei and Dong, Zhen and Bai, Jinbin and Singh, Aishani and Xiang, Xiaoyu and Li, Youzeng and Huang, Zuwei and Sun, Yuanxi and others},
  journal={arXiv preprint arXiv:2310.16003},
  year={2023}
}

@inproceedings{li2026flowdirector,
  title     = {{FlowDirector}: Training-Free Flow Steering for Precise Text-to-Video Editing},
  author    = {Li, Guangzhao and Yang, Yanming and Song, Chenxi and Zhang, Chi},
  booktitle = {IEEE/CVF Conference on Computer Vision and Pattern Recognition (CVPR)},
  year      = {2026}
}

@inproceedings{cong2024flatten,
  title     = {{FLATTEN}: Optical Flow-Guided Attention for Consistent Text-to-Video Editing},
  author    = {Cong, Yuren and Xu, Mengmeng and Simon, Christian and Chen, Shoufa and Ren, Jiawei and Xie, Yanping and Perez-Rua, Juan-Manuel and Rosenhahn, Bodo and Xiang, Tao and He, Sen},
  booktitle = {International Conference on Learning Representations (ICLR)},
  year      = {2024}
}

@article{zhao2023controlvideo,
  title     = {{ControlVideo}: Adding Conditional Control for One-Shot Text-to-Video Editing},
  author    = {Zhao, Min and Wang, Rongzhen and Bao, Fan and Li, Chongxuan and Zhu, Jun},
  journal   = {arXiv preprint arXiv:2305.17098},
  year      = {2023}
}

@article{liu2023videop2p,
  title={Video-p2p: Video editing with cross-attention control},
  author={Liu, Shaoteng and Zhang, Yuechen and Li, Wenbo and Lin, Zhe and Jia, Jiaya},
  journal={arXiv preprint arXiv:2303.04761},
  year={2023}
}

@article{meng2021sdedit,
  title={Sdedit: Guided image synthesis and editing with stochastic differential equations},
  author={Meng, Chenlin and He, Yutong and Song, Yang and Song, Jiaming and Wu, Jiajun and Zhu, Jun-Yan and Ermon, Stefano},
  journal={arXiv preprint arXiv:2108.01073},
  year={2021}
}

@inproceedings{yin2025causvid,
  title     = {From Slow Bidirectional to Fast Autoregressive Video Diffusion Models},
  author    = {Yin, Tianwei and Zhang, Qiang and Zhang, Richard and Freeman, William T. and Durand, Fr{\'e}do and Shechtman, Eli and Huang, Xun},
  booktitle = {IEEE/CVF Conference on Computer Vision and Pattern Recognition (CVPR)},
  year      = {2025}
}

@inproceedings{tumanyan2023pnp,
  title     = {Plug-and-Play Diffusion Features for Text-Driven Image-to-Image Translation},
  author    = {Tumanyan, Narek and Geyer, Michal and Bagon, Shai and Dekel, Tali},
  booktitle = {IEEE/CVF Conference on Computer Vision and Pattern Recognition (CVPR)},
  year      = {2023}
}

@inproceedings{song2023consistency,
  title     = {Consistency Models},
  author    = {Song, Yang and Dhariwal, Prafulla and Chen, Mark and Sutskever, Ilya},
  booktitle = {International Conference on Machine Learning (ICML)},
  year      = {2023}
}

@inproceedings{lipman2022flowmatching,
  title={Flow matching for generative modeling},
  author={Lipman, Yaron and Chen, Ricky TQ and Ben-Hamu, Heli and Nickel, Maximilian and Le, Matthew},
  booktitle={The eleventh international conference on learning representations},
  year={2022}
}

@inproceedings{liu2023rectifiedflow,
  title     = {Flow Straight and Fast: Learning to Generate and Transfer Data with Rectified Flow},
  author    = {Liu, Xingchao and Gong, Chengyue and Liu, Qiang},
  booktitle = {International Conference on Learning Representations (ICLR)},
  year      = {2023}
}


\end{document}